\documentclass{ecai} 

\usepackage{latexsym}
\usepackage{amssymb}
\usepackage{amsmath}
\usepackage{amsthm}
\usepackage{booktabs}
\usepackage{enumitem}
\usepackage{graphicx}
\usepackage{color}
\usepackage{algorithm}
\usepackage{algorithmic}

\usepackage{makecell}
\usepackage{multicol}
\usepackage{multicol} 
\usepackage{array} 
\usepackage{multirow} 

\newtheorem{theorem}{Theorem}

\newtheorem{corollary}[theorem]{Corollary}

\newtheorem{definition}{Definition}
\newtheorem{property}{Property}

\newcommand{\BibTeX}{B\kern-.05em{\sc i\kern-.025em b}\kern-.08em\TeX}

\begin{document}


\begin{frontmatter}


\paperid{123} 


\title{CogniSNN: An Exploration to Random Graph Architecture based Spiking Neural Networks with Enhanced Depth-Scalability and Path-Plasticity}


\author[A,B]{\fnms{Yongsheng}~\snm{Huang}}

\author[A]{%
  \fnms{Peibo}~\snm{Duan}%
  \thanks{Corresponding authors. Emails:
    duanpeibo@swc.neu.edu.cn (P.~Duan),
    xumingkun@gdiist.cn (M.~Xu).}%
}

\author[A]{\fnms{Zhipeng}~\snm{Liu}}
\author[C]{\fnms{Kai}~\snm{Sun}}
\author[A]{\fnms{Changsheng}~\snm{Zhang}}
\author[A]{\fnms{Bin}~\snm{Zhang}}
\author[B]{%
  \fnms{Mingkun}~\snm{Xu}%
  \footnotemark[*]   
}

\address[A]{School of Software, Northeastern University, Shenyang, China}
\address[B]{Guangdong Institute of Intelligence Science and Technology, Zhuhai, China}
\address[C]{Department of Data Science and AI, Monash University, Melbourne, Australia}

\begin{abstract}
Currently, most spiking neural networks (SNNs) still mimic the chain-like hierarchical architecture in traditional artificial neural networks (ANNs). This method significantly differs from random connections between neurons found in biological brains, limiting the ability to model the evolving mechanisms of neural pathways in biological neural systems, particularly in terms of dynamic \textbf{depth-scalability} and adaptive \textbf{path-plasticity}. This paper develops a new modeling paradigm for SNNs with random graph architecture (RGA), termed Cognition-aware SNN (CogniSNN). Furthermore, we model the depth-scalability and path-plasticity in CogniSNN by introducing a modified spiking residual neural node (ResNode) to counteract network degradation in deeper graph pathways, as well as a critical path-based algorithm that enables CogniSNN to perform path reusability on new tasks leveraging the features of the data and the RGA learned in old tasks. Experiments show that the performance of CogniSNN with redesigned ResNode is comparable, even superior, to current state-of-the-art SNNs on neuromorphic datasets. The critical path-based approach effectively achieves path reuse capability while maintaining expected performance in learning new tasks that are similar to or distinct from the old ones. This study showcases the potential of RGA-based SNNs and paves a new path for modeling the fusion of computational neuroscience and deep intelligent agents. The code is available at \url{github.com/Yongsheng124/CogniSNN}.
\end{abstract}

\end{frontmatter}


\section{Introduction}

Spiking neural networks (SNNs), inspired by biological nervous systems, provide enhanced biological interpretability and energy efficiency over traditional artificial neural networks (ANNs).
Recently, various prominent SNN models have been developed, such as Spiking VGG~\cite{2022_feng_multi}, Spiking ResNet~\cite{2021_fang_deep}, and Spiking Transformer~\cite{2022_zhou_spikformer}. These models maintain the performance benefits of conventional ANNs while leveraging the intrinsic energy efficiency of SNNs, achieving improved performance on neuromorphic datasets. However, these developments remain predominantly driven by deep learning paradigms, which significantly deviate from the original brain-inspired purpose of SNNs.

The human brain constructs extremely deep neural pathway networks through random connections among hundreds of billions of neurons, and this topology endows the nervous system with strong adaptability to process complex tasks~\cite{2023_shen_esl,2023_han_adaptive}. When encountering new task challenges, the brain always selectively and frequently activates specific pathways, enabling a seamless transition from simple task processing to complex reasoning~\cite{2021_barron_neural, 2023_han_adaptive}. In this paper, we define these two fundamental characteristics as \textbf{depth-scalability} (enhancing network capacity through pathway depth extension) and \textbf{path-plasticity} (enabling lifelong learning through dynamic path selection), which collectively constitute the foundation of biological intelligence's adaptive capabilities. Empirical evidence demonstrates these two mechanisms enable the brain to accomplish environmental adaptation and lifelong learning with ultra-low power consumption. However, existing SNNs have yet to effectively model these properties, whose architectures remain constrained by accuracy-oriented ANN paradigms, neglecting biologically validated efficient learning mechanisms. This divergence results in SNNs' deficiencies in biological interpretability, energy efficiency, and dynamic environmental adaptability. We believe that developing a novel computational framework integrating depth-scalability with path-plasticity represents a critical breakthrough to advance brain-like intelligence.

Current SNNs mainly adopt conventional chain-like hierarchical architectures~\cite{2017_xie_aggregated}, making it difficult to effectively model neural pathways. In contrast, the brain's biological structure resembles a random graph architecture (RGA)~\cite{1958_rosenblatt_perceptron,2011_varshney_structural}, and the modeling of neural pathways also requires the assistance of graph structures. Besides, it is crucial to recognize that if SNNs still inherit the technological framework of conventional ANN, it may encounter similar issues related to hyperparameter configuration in ANNs, such as the number of layers, leading to performance barriers~\cite{2023_xu_unified}. These insights urge a re-evaluation of the SNN framework through the RGA perspective. In practice, a similar effort was made by Xie et al., who performed experiments to validate that RGA-based ANNs achieved equivalent performance to computational graph-based ANNs in multiple datasets~\cite{2019_xie_exploring}. To our knowledge, the most relevant evidence comes from the investigation conducted by Yan et al., who performed a neural architecture search on SNN models~\cite{2024_yan_sampling}, indirectly highlighting the promise of RGA-based SNNs. Some studies~\cite{2021_xu_exploiting, 2023_xu_exploiting} combine graph learning and SNNs, but mainly focus on graph data rather than graph structure of the network. Despite the lack of direct studies on the feasibility of RGA-based SNNs, such intuition can be derived by comparing the phase of spike processing over all neurons in an SNN to graph signal processing, which is highly dependent on the graph structure. 

After simply evaluating the feasibility of RGA-based SNNs, it is eager to explore the potential of RGA-based SNNs in exhibiting \textbf{depth-scalability} and \textbf{path-plasticity}, which are two fundamental characteristics inherent in the architecture of the human brain, yet challenging to implement in RGA-based SNN models. Specifically, in terms of depth-scalability, first, a long path in the graph leads to gradient vanishing/exploding. Although ResNet-inspired SNN approaches could address this problem, they require significant real-valued computations~\cite{2024_zhou_direct}, undermining the biologically inspired nature of SNNs. Moreover, unlike computational graphs where connections exist solely between neurons from adjacent layers, an edge in an RGA-based SNN possibly links neurons across disparate layers. Consequently, it results in feature dimension misalignment, rendering pooling mechanisms ineffective.

In terms of path-plasticity, drawing from continual learning technologies, strategies to achieve this goal mainly include methods based on regularization~\cite{2024_xu_adaptive}, replay techniques~\cite{2021_chaudhry_using}, architectural adjustments~\cite{2025_shen_context}, and knowledge distillation~\cite{2025_li_frequency}. However, due to the limitations of conventional neural network architectures, similar architectural adjustment methods tend to focus on the internal structure of neurons within convolutional layers, neglecting the path structures in graph-structured networks. This makes these methods unsuitable for RGA-based SNNs. Moreover, the management of these paths becomes challenging due to task interference, a difficulty exacerbated when there is a significant domain shift between the source (old) tasks and the target (new) tasks.

This study introduces a Cognition-aware SNN model (CogniSNN), conceived with the notion of efficiently modeling the depth-scalability and path-plasticity cognition that is inherently present in the biological brain through RGA-based SNNs. We utilize Erd\H{o}s--R\'{e}nyi (ER) and Watts–Strogatz (WS) graph frameworks to structure CogniSNN, on which we develop a continual learning algorithm based on critical paths. The main contributions are as follows.
\begin{itemize}

    \item \textbf{Depth-Scalability}: An OR Gate and a tailored pooling mechanism are proposed, where the former addresses network degradation and non-spike transmission, whereas the latter corrects feature dimension misalignment.

    \item \textbf{Path-Plasticity}: We develop a critical path-focused method to improve the pathway reuse capability of RGA-based SNNs. By utilizing betweenness centrality to identify essential pathways, CogniSNN selectively retrains the parameters associated with these pathways, mitigating catastrophic forgetting while maintaining performance on new tasks.   
    \item Extensive experiments show that CogniSNN achieves performance comparable to or even surpassing current state-of-the-art models on all three datasets. In addition, our critical path-based path adaptive mechanism enables CogniSNN to manage both similar and dissimilar tasks, effectively suppressing catastrophic forgetting.
\end{itemize}

\section{Related Work}
\subsection{Residual Structures in SNNs}
\label{sec:ResInSNN}
Residual learning~\cite{2016_he_deep} is necessary for training large-scale deep neural networks. Spiking ResNet~\cite{2021_hu_spiking} introduces residual learning in SNNs for the first time, converting trained ResNet to an SNN by replacing ReLU layers with spike neurons. On the one hand, although it is better than SNNs without residual operation, the information of residual blocks is not effectively preserved due to the activation of spike neurons after the residual operation. Therefore, when the number of layers exceeds 18, deep Spiking ResNet begins to degrade gradually. On the other hand, it does not verify whether directly trained SNNs can address the degradation problem. To solve these two problems, SEW-ResNet~\cite{2021_fang_deep} places the residual operation after the spike neuron activation, and for the first time, directly trains a 152-layer SNN. However, due to the addition operation, output of the residual block is changed from spike form to integer form, which violates the principle of spike computation in SNNs, and leads to greater energy consumption. Moreover, the output is amplified infinitely as the layers deepen, leading to an extremely unstable training and computation process.
MS-ResNet~\cite{2024_hu_advancing} places spike neuron activation in front of the convolutional layer to minimize non-spike computation in the convolutional layer. However, the output of the residual block is still non-spike, which leads to floating-point signal transmission in the subsequent residual branching.

\subsection{Continual Learning in SNNs}
\label{sec:cl}
Continual learning enables networks to continually learn as the human brain and adapt to dynamic environments in the real world, with the same brain-like perspective as SNNs. Recently, researchers have begun to focus on continual learning in SNNs.~\cite{2022_antonov_continuous} propose a novel regularization method based on synaptic noise and Langevin dynamics, to preserve important parameters for the locally STDP-trained SNNs in continual learning.~\cite{2023_proietti_memory} achieve continual learning by using experience replay on SNNs for the first time, and perform well in class-incremental learning and task-free continual learning scenarios. Inspired by the structure of the human brain and learning process in the biological system,~\cite{2023_han_adaptive} propose a self-organizing regulation network that selectively activates different sparse pathways according to different tasks to learn new knowledge. Similarly,~\cite{2024_shen_efficient} utilize trace-based K-Winner-Take-All and variable threshold components to achieve selective activation of sparse neuron populations, not only in the spatial dimension but also in the temporal dimension for continual learning under specific tasks.~\cite{2023_han_enhancing} dynamically change the structure of SNNs to improve learning ability of the SNNs for new tasks by randomly growing new neurons and adaptively pruning redundant neurons.

Restricted by structure, some of the above approaches are not biologically plausible enough, some require extra memory usage, and some necessitate the introduction of the complex selection mechanisms at the spiking neuron layer, which greatly increases the complexity of SNN implementation.

\section{Preliminaries}
\subsection{Spiking Neuron}

SNNs employ neurons characterized by biological dynamics to more closely emulate brain functionality. The generation of spike in such neurons is typically divided into three stages: charging, discharging, and resetting. Upon receiving an input current, a spiking neuron accumulates it as membrane potential. The charging mechanism is mathematically represented by:

\begin{align}
    H[t] = f(V[t-1],C[t]), 
    \label{eq: charging}
\end{align}%

\noindent where $H[t]$ is the membrane potential before discharging at moment $t$, $C[t]$ refers to the input current at moment $t$, $V[t-1]$ is the membrane potential after discharging at moment $t-1$, and $f(\cdot)$ varies with the configuration of neurons. 

At any given time $t$, the output of a discharging process at a neuron depends on whether the membrane potential exceeds a threshold $V_{thr}$. This process is expressed as:
\begin{align}
    S[t]=\Phi\left(H[t]-V_{thr}\right)=\left\{\begin{array}{ll}
    1, & H[t] \geq V_{thr} \\
    0, & H[t]<V_{thr}
    \end{array} \right., 
    \label{eq: discharging}
\end{align}%
\noindent where $S[t]$ is the output at $t$, and $\Phi$ is the Heaviside step function.

Following the discharging process, the membrane potential experiences the resetting phase, which is categorized into soft and hard resetting. In this study, we employ the soft resetting approach, and this mechanism is represented by:
\begin{align}
    V[t] = H[t] - V_{thr} \cdot S[t].
\label{eq: reset}
\end{align}%

\subsection{Learning without Forgetting}
\label{sec:lwf}

Learning without forgetting (LwF)~\cite{2017_li_learning} is a continual learning approach without using old task data, and aims to solve catastrophic forgetting of old task while learning new task. 

Specifically, the method first computes the soft target $Y_{o}$ using new task data $X$, shared parameters $\theta_{s}$ and old task specific parameters $\theta_{o}$, where both parameters come from a well-trained model on old task. The process is mathematically represented by:
\begin{align}
    Y_{o}= CNN(X,\theta_{s},\theta_{o}). \label{eq: lwf1}
\end{align}%

During training of the new task, the LwF method requires using $\hat{\theta}_{o}$ and new task specific parameters $\hat{\theta}_{n}$ being trained to obtain old task output $Y'_{o}$ and new task output $Y'_{n}$, respectively. These two processes are shown by:
\begin{align}
Y'_{o}= CNN(X,\hat{\theta}_{s},\hat{\theta}_{o}), \quad Y'_{n}= CNN(X,\hat{\theta}_{s},\hat{\theta}_{n}). \label{eq: lwf2}
\end{align}%

Finally, a total loss $\mathcal{L}$ is utilized to balance learning and forgetting, defined as follows:
\begin{align}
\mathcal{L} = \lambda L_{old}(Y_{o},Y'_{o}) + L_{new}(Y_{n},Y'_{n})+R(\hat{\theta}_{s},\hat{\theta}_{o},\hat{\theta}_{n}).
\end{align}%
The $L_{new}(\cdot)$, computed using new task output $Y'_{n}$ and new task label $Y_{n}$, is used to ensure that the model learns the new task. Additionally, $L_{old}(\cdot)$, computed using the soft target $Y_{o}$ and the old task output $Y'_{o}$, is used to mitigate the model's forgetting of old tasks. $R(\cdot)$ is the regularization term introduced to prevent overfitting.

\section{Modeling and Methodology}

\begin{figure*}[t]
    \centering
    \includegraphics[width=\textwidth]{img/Framework.pdf}
    \caption{Comprehensive diagram of CogniSNN. During neuromorphic object recognition, data flows from the $ConvBNSN$ triple to the RGA-based SNN and then enters the classifier. The internal structure of ResNode and the illustration of dimension misalignment are shown at the bottom right and right side of the figure, respectively. During continual learning, different critical paths are selected for learning new knowledge in different scenarios.}
    \label{fig: Framework}
\end{figure*}

This study first concentrates on the implementation of the depth-scalability of RGA-based SNNs by introducing CogniSNN, facilitated by ResNode and a tailored pooling strategy. Subsequently, we detail CogniSNN's capability to manage diverse tasks through a novel critical-path continual learning approach.

\subsection{Modeling of CogniSNN}
\label{subsec: RGA}

As illustrated in Fig.\ref{fig: Framework}, CogniSNN is designed to tackle classification tasks by integrating a triplet component to extract preliminary features from the input, a principal RGA-based module, and a full-connection classifier layer.

The triplet component comprises a convolution layer (Conv), a batch normalization layer (BN), and a spiking neuron layer (SN). In the following, we denote this triplet by $ConvBNSN$.

The RGA-based module is modeled as a directed acyclic graph using WS or ER generators referenced in~\cite{2019_xie_exploring}. With $N$ denoted as the number of nodes, the graph is formulated as $G = \{\mathcal{V}, E\}$, where $\mathcal{V}$ is the set of nodes with $|\mathcal{V}| = N$, and $E$ is the set of edges. As depicted in the yellow block of Fig.\ref{fig: Framework}, each node $v_i \in \mathcal{V}$ represents a ResNode, a variant of the residual block used in~\cite{2021_hu_spiking}. The design of ResNode attempts to solve the issue that a long propagation easily leads to the gradient vanishing or explosion. The adjacency matrix of $G$ is $A$, in which $A_{i,j} = w, w \neq 0$ indicates that there exists an edge from $v_i$ to $v_j$ with weight $w$. Otherwise, $A_{i,j} = 0$. At time $t$, the input $I_{j}[t]$ to node $v_j$ from its predecessors $\mathcal{P}_j$ is given by the following:
\begin{align}
    I_{j}[t] = \sum_{i \in \mathcal{P}_j} \sigma(A_{i,j}) \times TP(SP(O_{i}[t]))
    \label{eq: I_j[t]},
\end{align}%

\noindent where $O_i[t]$ is the output of $v_i$, $\sigma(\cdot)$ is the sigmoid function, $TP$ and $SP$ stand for the tailored pooling and standard pooling. To avoid feature loss, $SP$ is defined as follows:
\begin{align}
    SP(O_{i}[t]) = \left\{\begin{array}{ll}
    AvgP(O_{i}[t], \kappa), & D(O_{i}[t]) \geq \eta \\
    O_{i}[t] & Otherwise
    \end{array}, \right.
\end{align}%

\noindent where $AvgP(\cdot)$ denotes average pooling, $D(\cdot)$ denotes the spatial dimension of a feature map, and $\eta$ is set to 1 and 32 in the classification and continual learning experiments, respectively. $\kappa$ represents the pooling kernel size and is generally set to 2. $TP(\cdot)$ aims to solve the issue of feature misalignment caused by random connections of the spike neuron nodes in RGA. The solutions to ResNode modeling and the tailored pooling mechanism are the keys to improving the depth-scalability of CogniSNN. Further elaboration is provided below.

\subsubsection{ResNode}
\label{subsubsec: resnode}

In greater detail, $v_i$ comprises two $ConvBNSN$ triplets, which are denoted as $ConvBNSN_{i}^{1}$ and $ConvBNSN_{i}^{2}$. $ConvBNSN_{i}^{1}$ is used to transform $I_i[t]$ back into spike form. This is necessary because, as noted in \eqref{eq: I_j[t]}, performing weighted operations converts $I_i[t]$ into a real-valued quantity. We then present a sophisticated skip connection that runs parallel to $ConvBNSN_{i}^{2}$ with the procedure mathematically expressed as follows.
\begin{equation}
    \begin{split}
        O_{i}^{1}[t] &= ConvBNSN_{i}^{1}(I_{i}[t]) \\[10pt] 
        O_{i}^{2}[t] &= ConvBNSN_{i}^{2}(O_{i}^{1}[t]) \\[10pt] 
        O_{i}[t] &= OR(O_{i}^{2}[t], O_{i}^{1}[t]) \\
                 &= O_{i}^{2}[t] + O_{i}^{1}[t] - O_{i}^{2}[t] \odot O_{i}^{1}[t]
    \end{split}.
    \label{eq:skip}
\end{equation}

In particular, $O_{i}^{1}[t]$ is derived from $ConvBNSN_{i}^{1}$ and is regarded as an identity. It is logically integrated with $O_{i}^{2}[t]$, which is obtained from $ConvBNSN_{i}^{2}$. Instead of an ADD gate, an OR gate ($OR(\cdot)$) is used because it enables the output of node $O_{i}[t]$ to maintain a spiking format rather than containing real values. Subsequently, $O_{i}[t]$ is transmitted to the descendant nodes of $v_i$. Residual learning facilitates the training of deep neural networks through identity mapping~\cite{2016_he_identity}. Similarly, we state the following theorem. 

\begin{theorem}
    The operation of an OR gate performs an identity mapping function.
\end{theorem}

\begin{proof}
    Given $v_i$, $O_{i}^{2}[t]$ specifies the residual mapping to be learned. By configuring the weights and biases of Batch Normalization (BN) in the second triplet to 0, or by making the attenuation constant and the activation threshold sufficiently large, we achieve $O_{i}^{2}[t]\equiv 0$. Consequently, the output from $v_i$ becomes $O_{i}[t] = OR(O_{i}^{2}[t], O_{i}^{1}[t]) = OR(0, O_{i}^{1}[t]) = O_{i}^{1}[t]$, which is uniformly applied to each node. Therefore, if $O_{i}^{1}[t] = 1$, then $O_{i}[t]$ is 1. Otherwise, $O_{i}[t]$ is 0 under the condition $O_{i}^{1}[t] = 0$. This illustrates how the OR gate effectively achieves identity mapping.
\end{proof}
\begin{corollary}
    OR gate operation can overcome gradient vanishing/explosion.
    \label{corollary1}
\end{corollary}
In the ResNode, when implementing identity mapping, this residual structure can overcome gradient vanishing/explosion. 
With $O_{i}^{2}\equiv 0$, the gradient of the output of ResNode relative to the input of the residual block $ConvBNSN_{i}^{2}$ (the output of $ConvBNSN_{i}^{1}$) can be calculated as:
\begin{align}
\frac{\partial O_{i}[t]}{\partial O_{i}^{1}[t]} =   \frac{\partial OR(O_{i}^{2}[t]  ,O_{i}^{1}[t])}{\partial O_{i}^{1}[t]} =\frac{\partial O_{i}^{1}[t]}{\partial O_{i}^{1}[t]}=1.
\end{align}%
We do not consider the $ConvBNSN_{i}^{1}$ when calculating the gradient, because the input of $ConvBNSN_{i}^{1}$ is weighted and summed by multiple predecessors. Thus, we allow some loss in this part. This illustrates why the OR gate can overcome network degradation.

Compared with ADD operation, the utilization of OR operation has the following property:
\begin{property}
    OR gate operation can solve the problems of non-spike calculation and infinite amplification.
    \label{property1}
\end{property}

\subsubsection{Tailored Pooling Mechanism}
\label{subsubsec: TP}

Given a node $v_j \in \mathcal{V}$ (not the root node), we estimate $D(I_j[t])$ according to the outputs from its predecessors $\mathcal{P}_j$ as follows:
\begin{align}
    D(I_j[t]) = \min \left \{ D(O_{i}[t])\mid i \in \mathcal{P}_j \right \}. 
\end{align}%

\noindent Consequently, $TP(\cdot)$ is defined as follows:
\begin{align}
    TP(O_{i}[t]) = AvgP(O_{i}[t],\frac{D(O_{i}[t])}{D(I_j[t])} ),
    \label{eq: TP}
\end{align}%

\noindent with the significance of implementing average pooling of $O_{i}[t]$ based on the kernel size $\frac{D(O_{i}[t])}{D(I_j[t])}$.

\subsection{Critical Path-Based LwF}
\label{sec:cp-cl}

To capture the path-plasticity of CogniSNN as elucidated in the Introduction, we focus on parameter fine-tuning associated with the selective substructures, termed critical paths to achieve the neural reuse to reshape functions in response to new experiences. This idea is inspired by visualizing the activation frequency of nodes. As shown in \textbf{Figure 9} of Supplementary Material \footnote{The supplementary material and source code are both available at \url{https://github.com/Yongsheng124/CogniSNN}.}, it can be observed that specific ResNodes and certain paths (defined as sequences of nodes in CogniSNN with information flow) are activated more/less frequently than others. Based on such observation, a critical path is targeted based on the concept of betweenness centrality from graph theory, which is commonly used to quantify the importance of a node (or edge) based on its role in connecting other nodes. In this paper, given a path $p$ formulated as:
\begin{align}
    p = \{ v_{1}, e_{1},v_{2},e_{2},...,e_{l_p},v_{l_p+1} \},
\end{align}%

\noindent where $l_p$ is the length of $p$, indicating the number of edges involved in $p$.

Then we define the betweenness centrality of $p$, denoted by $C_B(p)$ as follows.

\begin{definition}
    $C_B(p) = \sum_{i=1}^{l_p+1} C_{B}(v_{i}) + \sum_{j=1}^{l_p} C_{B}(e_{j}).$
    \label{eq: C_B(p)}
\end{definition}

In Definition \ref{eq: C_B(p)}, $C_{B}(v_{i})$ and $C_{B}(e_{j})$ are the betweenness centrality of a node and an edge, which follow the principle of graph theory \footnote{The definition of node betweenness centrality and edge betweenness centrality are in \textbf{Section 3} of Supplementary Material.}. Based on the definition of $C_B(p)$, we sort all paths in descending order according to their betweenness centrality:
\begin{align}
    P = \{ p_{1},p_{2},...,p_{|P|} \},
\end{align}%

\noindent where $C_{B}(p_{1})\ge C_{B}(p_{2}) ...\ge C_{B}(p_{|P|})$, $|P|$ is the total number of paths.

We then define the critical paths ($\mathcal{C}$) as a set of paths with capacity $K$ according to the definition of $C_B(p)$ and the concept related to the similarity between old and new tasks, mathematically formulated as follows.

\begin{definition}
    $\mathcal{C} \triangleq \{p_{k}| k \in [1, K], K \leq |P| \}$ if two tasks are similar. Otherwise, $\mathcal{C} \triangleq \{ p_{k} | k \in [|P|- K + 1, |P|], K \leq |P| \}$.
    \label{def: critical paths}
\end{definition}

This definition aligns with human cognitive processes, where principal neural pathways are leveraged for similar tasks, while unused connections are activated for unfamiliar tasks to assimilate new information and avoid forgetting prior knowledge. For simplicity, we set $K = 1$ in this study.

\begin{algorithm}[tb]
    \caption{The critical path-based LwF algorithm}
    \label{alg:algorithm}
    \textbf{Input}: $X^+$, $Y^+$, $CogniSNN$, $\Theta$, $\Theta^+$, $\mathcal{I}$\\
    \textbf{Output}: $CogniSNN^+$
    \begin{algorithmic}[1] 
        \STATE $CogniSNN \rightarrow CogniSNN^+$.
        \STATE Aligning dimensions of $\Theta$ and $\Theta^+$.
        \STATE Freeze all parameters in $CogniSNN^+$ except for $\theta_{\mathcal{C}}$ and $\theta_{\Theta^+}$.
        \FOR{$i \in \mathcal{I}$}
            \FOR{$\forall (x^+, y^+)$ in $(X^+, Y^+)$}
            \STATE Update $\theta_{\mathcal{C}}$ and $\theta_{\Theta^+}$ with $(x^+, y^+)$ using LwF.
            \ENDFOR
        \ENDFOR
        \STATE \textbf{return} $CogniSNN^+$
    \end{algorithmic}
\end{algorithm}

We design a continual learning algorithm implemented on CogniSNN by integrating critical paths and the LwF method outlined in Section \ref{sec:lwf}. The details are given in Algorithm ~\ref{alg:algorithm}. More precisely, $X^+$ and $Y^+$ are the dataset and label set in the new task. $CogniSNN$ is the model pre-trained based on the dataset in the old task. $\Theta$ and $\Theta^+$ are classifier layers in $CogniSNN$ and $CogniSNN^+$, respectively. Here, $CogniSNN^+$ denotes the retrained $CogniSNN$ based on $(X^+, Y^+)$. $\mathcal{I}$ represents the number of learning rounds for new model. In line 2, we align the dimensions of $\Theta$ and $\Theta^+$ by adding neurons in $\Theta$. In line 3, we freeze all parameters except $\theta_{\mathcal{C}}$ and $\theta_{\Theta^+}$ that are associated with $\mathcal{C}$ and $\Theta^+$. From lines 4 to 8, LwF is used to refine $\theta_{\mathcal{C}}$ and $\theta_{\Theta^+}$ with the data in the new task.

To measure the similarity between the datasets, we use the Fréchet Inception Distance (FID)~\cite{2017_heusel_gans}, which is originally designed to evaluate generative models by comparing the differences in the distribution of the generated and real images in the feature space. Based on empirical observations, we consider two datasets similar when their FID is below the threshold of 50. For example, CIFAR100 and CIFAR10 show strong similarity with an FID of \textbf{24.5}, while CIFAR100 and MNIST, having an FID of \textbf{2341.2}, show weak similarity. 

Theoretically, this critical path can be combined with many basic continual learning methods. In this paper, we only apply the more biologically plausible LwF approach for initial validation and exploration. In the future, we will integrate more continual learning methods to explore the critical path more richly.

\section{Experiments}

\subsection{Experimental Setting}
\begin{table*}[ht]
\caption{Comparison with other benchmark results on DVS-Gesture, CIFAR10-DVS, N-Caltech101 and Tiny-ImageNet.}
    \centering
    \renewcommand{\arraystretch}{1.2} 
    \setlength{\tabcolsep}{0pt} 
    \begin{tabular*}{\textwidth}{@{\extracolsep{\fill}}>{\centering\arraybackslash}p{0.14\textwidth}
                                                    >{\centering\arraybackslash}p{0.17\textwidth}
                                                    >{\centering\arraybackslash}p{0.17\textwidth}
                                                    >{\centering\arraybackslash}p{0.09\textwidth}
                                                    >{\centering\arraybackslash}p{0.10\textwidth}
                                                    >{\centering\arraybackslash}p{0.14\textwidth}}
        \toprule
        Dataset & Method & Network & Param(M) & T & Accuracy($\%$) \\
        \midrule
        \multirow{5}{*}{\makecell{DVS-Gesture\\\cite{2017_amir_low}}}
                           & MLF~\cite{2022_feng_multi} & VGG-9 & - & 5 & 85.77 \\
                           & SEW-ResNet~\cite{2021_fang_deep} & SEW-ResNet & 0.13 & 16 & 97.90 \\
                           & SGLFormer~\cite{2024_zhang_sglformer} & SGLFormer-3-256 & 2.17 & 16 & 98.60 \\
                           & CML~\cite{2023_zhou_enhancing} & Spikformer-4-384 & 2.57 & 16 & 98.60 \\
                           & Spikformer~\cite{2022_zhou_spikformer} & Spikformer-2-256 & 2.57 & 5 / 16 & 79.52 / 98.30 \\
                           & SSNN~\cite{2024_ding_shrinking}& VGG-9 & - & 5 / 8 & 90.74 / 94.91 \\

         \cmidrule{2-6}
                           & \multirow{3}{*}{\textbf{CogniSNN(Ours)}} & \multirow{3}{*}{\textbf{ER-RGA-7}} &\multirow{3}{*}{\textbf{0.13}} & \textbf{5} &  \textbf{94.78}$ \hspace{0.1em} \scriptstyle \pm \hspace{0.1em} 0.12$\\
                           & &  && \textbf{8} & 
                           \textbf{95.81}$ \hspace{0.1em} \scriptstyle \pm \hspace{0.1em} 0.15$\\
                           & &&  & \textbf{16} & 
                            \textbf{98.61}$ \hspace{0.1em} \scriptstyle \pm \hspace{0.1em} 0.11$ \\

        \cmidrule{1-6}
        
        \multirow{5}{*}{\makecell{CIFAR10-DVS\\\cite{2017_li_cifar10}}}
                           & Spikformer~\cite{2022_zhou_spikformer} & Spikformer-2-256 & 2.57 & 5 & 68.55\\
                            & MLF~\cite{2022_feng_multi} & VGG-9 & - & 5 & 67.07 \\
                            & AutoSNN~\cite{2024_shi_towards} & AutoSNN & - & 8 & 72.50 \\
                            & SEW-ResNet~\cite{2021_fang_deep} & SEW-ResNet & 1.19 & 8 & 70.20 \\
                           & SSNN~\cite{2024_ding_shrinking} & VGG-9 & - & 5 / \textbf{8} & 73.63 / \textbf{78.57} \\
        \cmidrule{2-6}
                           & \multirow{2}{*}{\textbf{CogniSNN(Ours)}} & \multirow{2}{*}{\textbf{ER-RGA-7}} & \multirow{2}{*}{\textbf{1.51}} & \textbf{5} & \textbf{75.82}$ \hspace{0.1em} \scriptstyle \pm \hspace{0.1em} 0.10$ \\
                           & & & & \textit{8} &  \textit{76.20}$ \hspace{0.1em} \scriptstyle \pm \hspace{0.1em} 0.20$\\
                
        \cmidrule{1-6}
      
        \multirow{5}{*}{\makecell{N-Caltech101\\\cite{2017_li_cifar10}}}
                           & MLF~\cite{2022_feng_multi} & VGG-9 & - & 5 & 70.4 \\
                           & Spikformer~\cite{2022_zhou_spikformer} & Spikformer & 2.57 & 5 & 72.8 \\
                           & EventMix \cite{2023_shen_eventmix} & ResNet-18 & - & 10 & 79.5 \\
                           & TIM \cite{2024_shen_tim} &Spikformer & - & 10 & 79.0 \\
                           & SSNN~\cite{2024_ding_shrinking} & VGG-9 & - & 5 / 8 & 77.97/ 79.25 \\
        \cmidrule{2-6}
                           & \multirow{2}{*}{\textbf{CogniSNN(Ours)}} & \multirow{2}{*}{\textbf{ER-RGA-7}} & \multirow{2}{*}{\textbf{1.51}} & \textbf{5} & \textbf{80.64}$ \hspace{0.1em} \scriptstyle \pm \hspace{0.1em} 0.15$\\
                           & & & & \textbf{8} &  \textbf{79.32}$ \hspace{0.1em} \scriptstyle \pm \hspace{0.1em} 0.15$\\
        \cmidrule{1-6}
      
        \multirow{5}{*}{\makecell{Tiny-ImageNet\\\cite{2015_le_tiny}}}
                           & Joint-SNN~\cite{2023_guo_joint} & VGG-16 & - & 4 & 55.39 \\
                           & SNASNet-Bw~\cite{2022_kim_neural} & Searched & - & 5 & 54.60 \\
                           & Online LTL~\cite{2022_yang_training} & VGG-13 & - & 6 & 55.37 \\
        \cmidrule{2-6}
                           & \textbf{CogniSNN(Ours)} & \textbf{ER-RGA-7} & \textbf{1.52} & \textbf{4} & \textbf{55.41}$ \hspace{0.1em} \scriptstyle \pm \hspace{0.1em} 0.17$\\
        \bottomrule
    \end{tabular*}
    \label{tab:totalTable}
\end{table*}

We conduct experiments five times on three neuromorphic datasets and one static dataset on the NVIDIA RTX 4090 GPU by SpikingJelly~\cite{2023_fang_spikingjelly}. Furthermore, to demonstrate the advanced performance of CogniSNN, we engage in a detailed comparative evaluation against the SOTA SNN models, which are distinguished by architectures such as chain-like connectivity, transformer-based, and ResNet-based designs. More details of the datasets and the hyperparameter configuration are provided in \textbf{Section 1} of Supplemental Material.    

\subsection{Preliminary Analysis}

Before deep analysis of the performance of CogniSNN, we first assess the feasibility and potential continual learning capacity of RGA-based SNNs by comparing CogniSNN with a chain-structured SNN. Both are arranged with an identical number of ResNodes. This preliminary analysis aims to motivate an in-depth exploration of RGA-based SNN modeling.

As for the potential of depth-scalability, CogniSNN configured with 7 ResNodes consistently outperforms chain-structured SNNs, as seen in Figure \ref{fig:bar}. Random connections among ResNodes act as implicit skip connections, dynamically prioritizing features by adjusting weights. This overcomes the limitation of sequentially increasing the dimensions of features with layers in traditional neural networks.
\begin{figure}[h]
    \centering
 \includegraphics[width=0.48\textwidth]{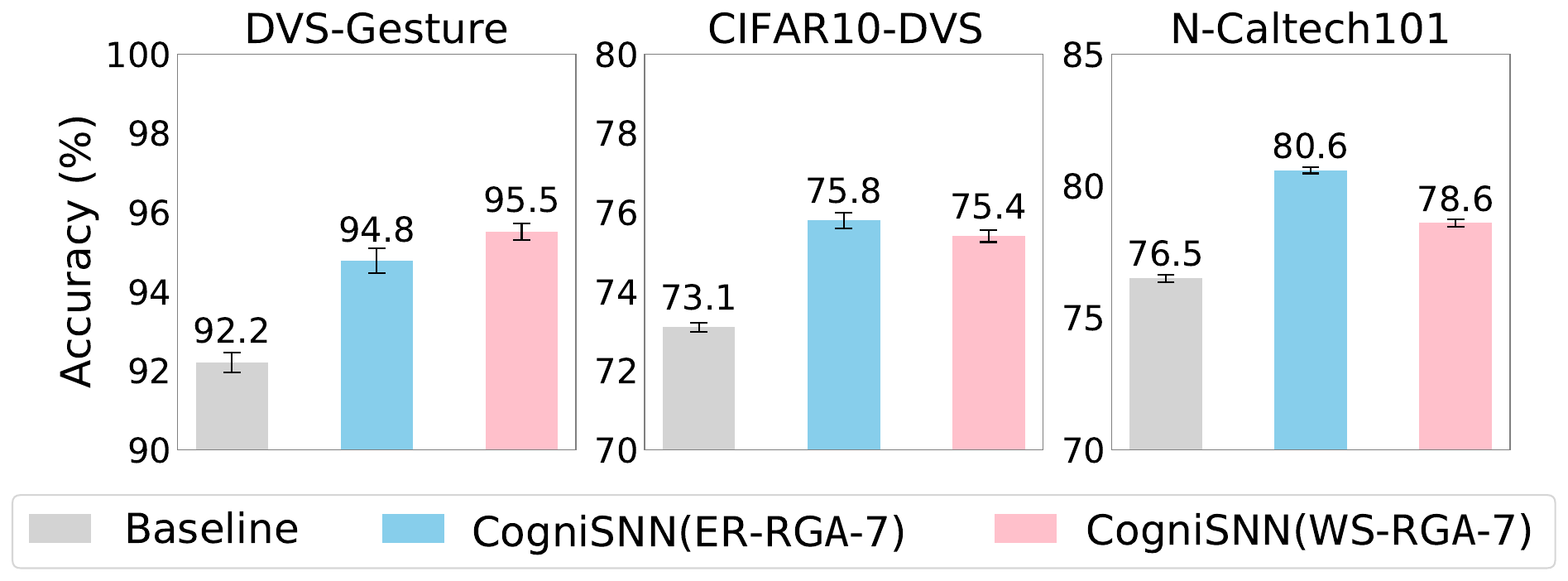}
 \vspace{-10pt}
    \caption{Comparison of results on three neuromorphic datasets between CogniSNN and the chain-like structured SNN.}
    \label{fig:bar}
\end{figure}
Regarding path-plasticity, under the implementation of the vanilla LwF algorithm, we compare CogniSNN with 32 ResNodes against a chain-like structure-based SNN. Table \ref{tab:CL1} and \ref{tab:CL2} show the potential of CogniSNN for continual learning in similar tasks and dissimilar tasks, respectively. In particular, CIFAR100 is treated as source data and CIFAR10 and MNIST are treated as target data. The last column indicates the classification accuracy in the source data, serving as a benchmark for comparison. Our CogniSNN exhibits exceptional classification capabilities for new tasks, achieving a minimal accuracy decline, notably a $48\%$ improvement in the MNIST dataset over the chain-like model and a $18\%$ increase relative to the benchmark. This illustrates the feasibility of RGA-based SNNs in the realm of continual learning.
\begin{table}[htbp]
    \centering
    \caption{Path-plasticity's potential for CogniSNN in similar tasks.}
    \begin{tabular}{lrr|r}
        \toprule
         \multirow{2}{*}{Model}  & CIFAR100  & CIFAR10 & CIFAR100\\
         \cmidrule{2-4}
        & source & target & benchmark    \\
        \midrule
        chain-like      & 46.3(-14.9)   &31.0  & 61.2  \\
        CogniSNN       & \textbf{50.4(-14.0)}    &\textbf{46.1}& 64.4  \\

        \bottomrule
    \end{tabular}
    \label{tab:CL1}
\end{table}%

\begin{table}[htbp]
    \centering
    \caption{Path-plasticity's potential for CogniSNN in dissimilar tasks.}
    \begin{tabular}{lrr|r}
        \toprule
       \multirow{2}{*}{Model}  & CIFAR100  & MNIST & CIFAR100\\
         \cmidrule{2-4}
        & source & target  & benchmark    \\
        \midrule
        chain-like     & 44.6(-16.6)   &34.8 & 61.2    \\
        CogniSNN     & \textbf{47.5(-16.9)}     &\textbf{82.3}& 64.4  \\

        \bottomrule
    \end{tabular}
    \label{tab:CL2}
\end{table}%

\subsection{Insightful Performance Analysis}
\label{subsec: feasibility}

Table \ref{tab:totalTable} presents the results of the experiments conducted with CogniSNN and comparative models. Clearly, when tested with the time step $T = 5$, our methodology outperforms the SSNN model by $4.04\%$ on DVS-Gesture, $2.19\%$ on CIFAR10-DVS and $2.67\%$ on N-Caltech101. 

However, at $T = 8$, it should be noted that the performance on the CIFAR10-DVS is marginally inferior. In this regard, we believe that the vanilla convolution module in ResNode is an important factor for the model performance. In future work, we will explore more advanced modules like self-attention~\cite{2017_vaswani_attention} to replace the convolution here, which we believe can achieve better performance.

With $T = 16$, transformer-based SNNs (e.g., CML) demonstrate strong performance in datasets such as DVS-Gesture. However, our model achieves comparable accuracy while utilizing significantly fewer parameters. For example, when evaluated on DVS-Gesture, our model requires only 1/20 of the parameters needed by the CML model (0.13 million versus 2.57 million).

On Tiny-ImageNet, we also achieved comparable results to the Joint-SNN model in 4 time steps, demonstrating that CogniSNN performs well when facing large static datasets.

These improvements validate that our model not only exhibits significantly reduced latency, but also maintains accuracy even with a compact parameter scale, regardless of configuration in the time steps, demonstrating its superior adaptability and effectiveness across a diverse range of tasks.

\begin{figure}[t]
    \centering
    \includegraphics[width=0.40\textwidth]{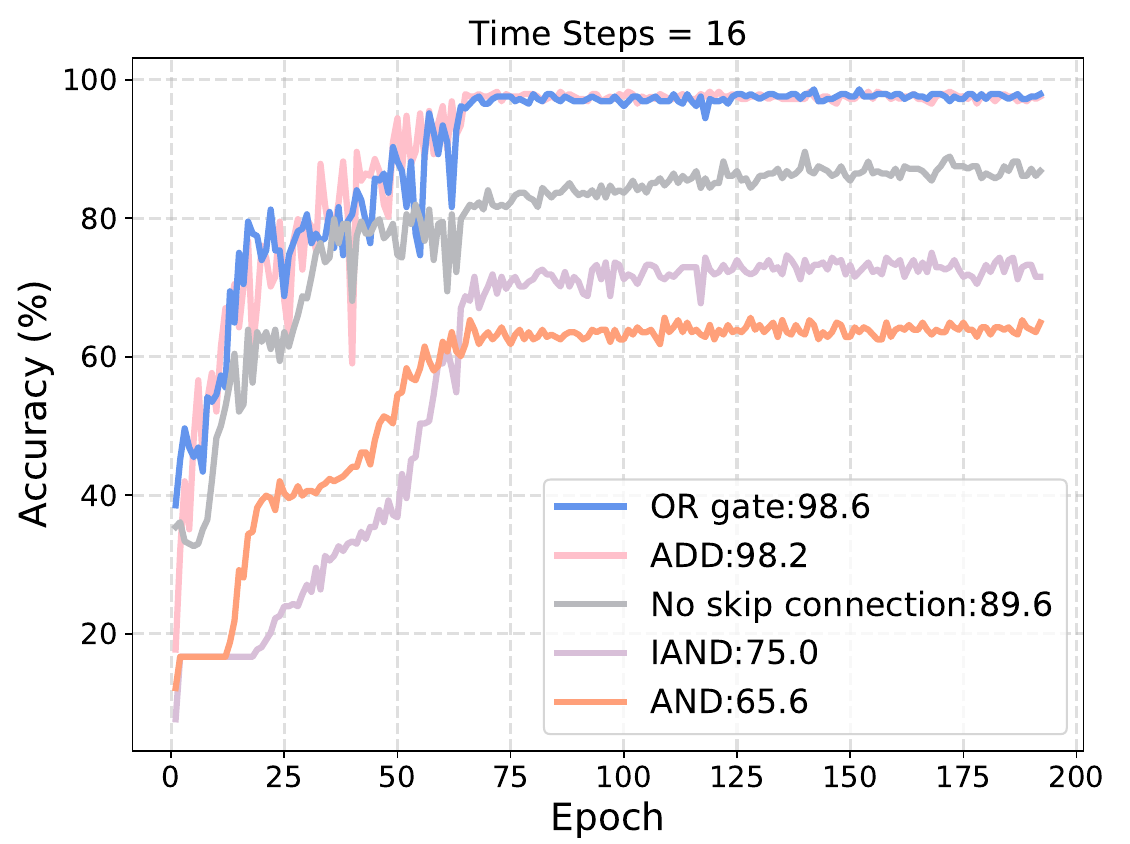}  
    \caption{Comparison of results of OR and other operations on DVS-Gesture with ER-RGA-7 based CogniSNN.}
    \label{fig:ablation1-2}
\end{figure}

\subsection{Depth-Scalability}
The depth-scalability of CogniSNN is heavily based on the usage of skip connections within ResNodes. As described in Section \ref{sec:ResInSNN}, early research is generally based on the following four configurations: 1) without skip connection, 2) ADD (addition), 3) AND, and 4) IAND. We compare our skip connection (OR gate) with the four aforementioned configurations based on the ER-driven CogniSNN. Furthermore, all models with/without configuring different skip connection mechanisms are trained on the DVS-Gesture.

The experimental results are shown in Fig.\ref{fig:ablation1-2}, which involve both the convergence speed and classification accuracy. It is obvious that our proposed model with an OR gate achieves a higher accuracy $9\%$ compared to the variant without using any skip connection mechanism. This discrepancy in the accuracy measurements can be attributed to the gradient vanishing or explosion, which occurs as a result of the deep path involved in information propagation. Intriguingly, the OR gate performs similarly to the ADD operation. However, the OR gate guarantees that the outputs of ResNode are solely spike signals that will not be infinitely amplified.

To explore the maximum scale of CogniSNN, we found that the network using 40 ResNodes with OR gate chain-like connected  does not show significant degradation, achieving 88.3$\%$ (versus 92.2$\%$ with 7 ResNodes) accuracy on the DVS-Gesture (T=5). Although this performance is not comparable to the 152 layers in SEW-ResNet with the ADD operation, it avoids more real-valued computations. Therefore, the maximum scale should be that the maximum path length in CogniSNN should not exceed 40 ResNodes.

\subsection{Path-Plasticity}

We apply the critical path-based LwF with $K = 1$ on the WS-driven CogniSNN configured with 32 ResNodes. Comparable results are observed with the ER-driven CogniSNN and are provided in \textbf{Section 6} of Supplementary Material. Tables \ref{tab:CL3} and \ref{tab:CL4} present the results, maintaining the significance of each column consistent with the corresponding column in Tables \ref{tab:CL1} and \ref{tab:CL2}. The first row presents results from the application of the fine-tuned CogniSNN using vanilla LwF on both tasks. The subsequent two rows display results utilizing the proposed critical path-based LwF, implemented based on paths with the highest and lowest betweenness centrality, respectively. 

\begin{table}[htbp]
    \centering
    \caption{Comparison of accuracy ($\%$) between similar tasks.}
    \begin{tabular}{lrr|r}
        \toprule

        \multirow{2}{*}{Method} &CIFAR100& CIFAR10 & CIFAR100\\
        \cmidrule{2-4}
           & source  & target  &benchmark   \\
        \midrule
        vanilla LwF     &44.9(-19.5)   &48.0 &\multirow{3}{*}{64.4}  \\
        high $C_{B}$ path         & \textbf{48.5(-15.9)}      &\textbf{48.0} & \\
        low $C_{B}$ path       & 48.5(-15.9)       &47.6\\
        \bottomrule
    \end{tabular}
    \label{tab:CL3}
\end{table}
\begin{table}[htbp]
    \centering
    \caption{Comparison of accuracy ($\%$) between dissimilar tasks.}
    \begin{tabular}{lrr|r}
        \toprule
        \multirow{2}{*}{Method}  &CIFAR100& MNIST & CIFAR100\\
        \cmidrule{2-4}
           &source & target  & benchmark    \\
        \midrule
        vanilla LwF      & 42.2(-22.2)   &83.9 &\multirow{3}{*}{64.4}   \\
        high $C_{B}$ path      & 46.6(-17.8)       &85.9\\
        low $C_{B}$ path     & \textbf{47.3(-17.1)}      &\textbf{86.1}\\
        \bottomrule
    \end{tabular}
    \label{tab:CL4}
\end{table}

In Table \ref{tab:CL3}, the performance of CogniSNN with fine-tuned parameters associated with the path in high $C_B$ is marginally more effective than fine-tuned based on low $C_B$, showing an improvement of $0.4\%$. Similarly, the performance of low $C_B$ improves by $0.2\%$ than that of high $C_B$ in Table \ref{tab:CL4}. The improvement is minor due to the simplicity of the classification task in MNIST and CIFAR10. However, this supports our hypothesis regarding the correlation between the selection of the critical path and the similarity of the task (Definition \ref{def: critical paths}). 

Relative to the benchmark, the level of forgetting is reduced. When compared specifically to the LwF-based method, our algorithm demonstrates a $3.6\%$ and $5.1\%$ reduction in forgetting on the previous task. This confirms that the critical path-based LwF can effectively mitigate catastrophic forgetting.

\subsection{Energy Analysis}

The energy consumption advantage of SNNs over ANNs for the same tasks has been demonstrated in many studies~\cite{2025_sun_ilif}. In this paper, we propose a spiking residual operation of the OR Gate to reduce the need for real-valued calculations, as well as a complex SNN architecture where ResNodes are randomly connected. Therefore, it is essential to analyze their energy consumption.

\begin{table}[htbp]
    \centering
    \caption{The total energy consumption for different tasks.}
    \begin{tabular}{lllll}
        \toprule
        &  \multirow{2}{*}T &   CogniSNN & CogniSNN & chain-like SNN \\
        & & (OR) & (ADD) & (OR) \\
        \midrule
         \multirow{3}{*}{DVS-Gesture} & 5 & \textbf{7.46 mJ}  & 8.89 mJ & \underline{\textit{7.75 mJ}}   \\
                                    & 8 & \textbf{7.86 mJ}  & 9.02 mJ & \underline{\textit{7.91 mJ}}  \\
                                   & 16 & \underline{\textit{8.61 mJ}} & 9.72 mJ & \textbf{8.43 mJ}  \\
         \cmidrule{1-5}
        \multirow{2}{*}{CIFAR10-DVS} & 5 & \underline{\textit{133.62 mJ}} & 135.32 mJ & \textbf{132.11 mJ} \\
                                    & 8 &\underline{\textit{118.92 mJ}} & 126.07 mJ & \textbf{115.13 mJ} \\
         \cmidrule{1-5}
        \multirow{2}{*}{N-Caltech101}  & 5 & \underline{\textit{113.69 mJ}} & 115.94 mJ & \textbf{110.88 mJ} \\
                                        & 8 & \underline{\textit{101.53 mJ}} & 114.15 mJ &  \textbf{98.56 mJ} \\
        \bottomrule
    \end{tabular}
    \label{tab:energy}
\end{table}

Table~\ref{tab:energy} presents the energy consumption of a CogniSNN with OR Gate, a CogniSNN with ADD operation and a chain-like SNN with OR Gate (which is chain-like connected by the same number of ResNodes), when testing a single sample from DVS-Gesture, CIFAR10-DVS, N-Caltech101 at different time steps.

It is evident that the OR Gate operation consistently consumes less energy than the ADD operation. This is primarily because the OR Gate maintains the output in spike form, thereby preventing infinite amplification and reducing the firing rate of neurons in deeper layers. Additionally, CogniSNN with RGA sometimes consumes slightly more energy than the chain-like SNN, which is intuitive given that the RGA involves a greater computational load. However, this energy consumption gap is acceptable, as the RGA also provides significant performance gains.

\section{Conclusion}

This paper validates the feasibility of RGA-based SNNs for modeling the inherent depth scalability and path plasticity in the brain and develops a neuroscience-inspired CogniSNN using the proposed ResNode, the OR gate-driven skip connection mechanism, and the critical path-based LwF algorithm. The results are surprising: Compared to traditional SNNs, the CogniSNN is competitive in image recognition and continual learning. This research opens numerous exciting avenues, including leveraging the scalability of CogniSNN for continual learning by automatically adjusting its structure, or optimizing the graph structure to enhance performance.

\newpage
\section{Acknowledgements}
This study is supported by Northeastern University, Shenyang, China (02110022124005) and Guangdong S\&T Program (2021B0909060002).


\bibliography{mybibfile}

\end{document}